\documentclass[letterpaper, 10 pt, conference]{ieeeconf}  

\IEEEoverridecommandlockouts                              

\usepackage{amsmath,amssymb,amsfonts,amsthm,dsfont} 
\usepackage{graphics}
\usepackage[dvips]{graphicx}
\usepackage[table,usenames,dvipsnames]{xcolor}
\usepackage{subcaption}
\usepackage{tikz-cd}
\usepackage{algorithm2e}
\usepackage[breaklinks=true, colorlinks, bookmarks=true, citecolor=Black, urlcolor=Violet,linkcolor=Black]{hyperref}

\newtheorem{lemma}{Lemma}
\newtheorem{proposition}{Proposition}

\theoremstyle{definition}
\newtheorem{definition}{Definition}
\newcommand{\pa}{\partial}
\newcommand{\ba}{\begin{align}}
\newcommand{\ea}{\end{align}}
\newcommand{\fr}{\frac}

\newcommand{\R}{{\mathbb R}}

\newcommand{\N}{{\mathcal N}}
\DeclareMathOperator*{\tr}{tr}
\DeclareMathOperator*{\diag}{diag}

\newcommand{\calF}{{\cal F}}

\newcommand{\calN}{{\cal N}}

\newcommand{\bfa}{\mathbf{a}}

\newcommand{\bfe}{\mathbf{e}}

\newcommand{\bfh}{\mathbf{h}}

\newcommand{\bfm}{\mathbf{m}}

\newcommand{\bfp}{\mathbf{p}}
\newcommand{\bfq}{\mathbf{q}}

\newcommand{\bfu}{\mathbf{u}}
\newcommand{\bfv}{\mathbf{v}}

\newcommand{\bfz}{\mathbf{z}}

\newcommand{\bfeta}{\boldsymbol{\eta}}

\newcommand{\bfmu}{\boldsymbol{\mu}}
\newcommand{\bfnu}{\boldsymbol{\nu}}

\newcommand{\bfomega}{\boldsymbol{\omega}}
\newcommand{\bfxi}{\boldsymbol{\xi}}

\newcommand{\bfU}{\mathbf{U}}

\newcommand{\bbH}{\mathbb{H}}

\newcommand{\bbR}{\mathbb{R}}
\newcommand{\bbS}{\mathbb{S}}

\title{\LARGE \bf Active Exploration and Mapping via Iterative Covariance Regulation over Continuous $SE(3)$ Trajectories}

\author{Shumon Koga \and Arash Asgharivaskasi \and Nikolay Atanasov
\thanks{We gratefully acknowledge support from ARL DCIST CRA W911NF-17-2-0181 and ONR SAI N00014-18-1-2828.}%
\thanks{The authors are with the Department of Electrical and Computer Engineering, UC San Diego, 9500 Gilman Drive, La Jolla, CA, 92093-0411, {\tt\small \{skoga,aasghari,natanasov\}@ucsd.edu}.}
\thanks{Supplementary materials can be found at \url{https://shumon0423.github.io/IROS2021_webpage}.}
}

\begin{document}

\maketitle
\thispagestyle{empty}
\pagestyle{empty}


\begin{abstract}
This paper develops \emph{iterative Covariance Regulation} (iCR), a novel method for active exploration and mapping for a mobile robot equipped with on-board sensors. The problem is posed as optimal control over the $SE(3)$ pose kinematics of the robot to minimize the differential entropy of the map conditioned the potential sensor observations. We introduce a differentiable field of view formulation, and derive iCR via the gradient descent method to iteratively update an open-loop control sequence in continuous space so that the covariance of the map estimate is minimized. We demonstrate autonomous exploration and uncertainty reduction in simulated occupancy grid environments. 
\end{abstract}


\section{Introduction}

Simultaneous Localization and Mapping (SLAM) is a key research direction that has enabled robots to transition from controlled, structured, and fully known environments to operation in a priori unknown real-world conditions \cite{cadena2016past}. Many current SLAM techniques, however, remain \emph{passive} in their utilization of sensor data. Active SLAM \cite{delmerico2017active} is an extension of the SLAM problem which couples perception and control, aiming to acquire more information about the environment and reduce the uncertainty in the localization and mapping process. 
Active SLAM introduces unique challenges related to keeping the map and location estimation process accurate, and yet computing and propagating uncertainty over many potential sensing trajectories efficiently to select an informative one. Most of the literature in active exploration and mapping focuses on discrete \cite{charrow2015information, zhang2020fsmi} or sampling-based \cite{hollinger2014sampling} planning techniques. However, the trajectories and information collection process of the robot sensors evolve continuously over the $SE(3)$ space of sensor poses. As evidenced by successful applications of continuous control to exploration in reinforcement learning \cite{dalal2018safe} and active target tracking \cite{zhou2011multirobot}, developing active SLAM techniques for continuous control is expected to reduce uncertainty more effectively and smoothly compared to discrete control techniques.


This paper develops a new forward-backward gradient computation technique to optimize multi-step control input sequences over the $SE(3)$ pose kinematics, leading to maximum uncertainty reduction. The core problem is formalized as $SE(3)$ trajectory optimization to minimize the differential entropy of the map state conditioned on the sequence of measurements obtained by an on-board sensor (e.g., Lidar or RGB-D camera). Assuming a Gaussian prior over the map state and its Bayesian evolution along a sensing trajectory, the differential entropy of the map is proportional to the log determinant of the covariance matrix at the final time. To ensure that the covariance matrix evolution is differentiable with respect to the control sequence, we introduce a new differentiable field of view formulation for the sensing model, providing a smooth transition from unobserved to observed space in the environment. Finally, the gradient of the objective function with respect to the multi-step control input sequences is computed explicitly and the control trajectory is updated via gradient descent. 

Throughout this paper, we focus on occlusion-free planning by allowing the robot to pass through the occupied space in the environment, which can be implemented for 2-D mapping using an aerial robot and have been studied in literature \cite{indelman2015planning}. While in \cite{indelman2015planning} the gradient descent is applied via perturbation method to obtain an approximated gradient, our approach develops an explicit gradient formulation. We emphasize that our formulation is different from gradient descent methods that optimize the instantaneous sensing cost \cite{grocholsky2002information} because it considers the sensing performance over a long-horizon in an optimal control formulation.


\textbf{Related Work:} One of the earliest approaches for active exploration and mapping \cite{yamauchi1998frontier} is based on detecting and planning a shortest path to map frontiers (boundaries between explored and unexplored space). Frontier-based exploration is an efficient and effective method, prompting its widespread use in robotics \cite{burgard2005coordinated,holz2010evaluating,oleynikova2016continuous}. Its purely geometric nature, however, is a limiting factor for considering sensing noise and map uncertainty reduction from a probabilistic inference perspective \cite{thrun2000probabilistic}. Information-theoretic planning is an alternative approach, which utilizes an information measure, such as mutual information or conditional entropy, to quantify and minimize the uncertainty in the map state. Information-theoretic mapping was first introduced by Elfes \cite{elfes1995robot}, and subsequently has been developed in many studies \cite{moorehead2001autonomous, grocholsky2002information, hollinger2014sampling, sommerlade2008information}, including applications to active SLAM \cite{carlone2014active,atanasov2015decentralized}. Evaluating information measures accompanies a high computational effort in general, which makes online planning challenging. Efficient computation methods have been proposed in \cite{charrow2015information} for Causchy-Schwarz quadratic mutual information (CSQMI), and in \cite{zhang2020fsmi} for fast Shannon mutual information (FSMI). Instead of binary occupancy grid mapping, recent information-based active mapping techniques have considered truncated signed distance field (TSDF) maps \cite{saulnier2020information} and multi-category semantic maps \cite{asgharivaskasi2021active}. Existing methods are, however, limited to discrete control spaces, typically with a finite number of possible control inputs \cite{Kantaros_InformationGathering_RSS19, saulnier2020information, charrow2015information, zhang2020fsmi}, and have not considered optimal control formulations of the active mapping problem.


Optimal control has been intensively studied since the work of Bellman \cite{bellman1952theory}, which developed the well-known dynamic programming algorithm. Applying dynamic programming to continuous control spaces requires finite-dimensional approximations of the value function. Instead of a globally optimal control policy, a locally optimal state-control trajectory may be obtained via Pontryagin's Maximum Principle (PMP) \cite{pontryagin2018mathematical}. An iterative approach to solve PMP was originally developed by Kelley and Bryson \cite{kelley1960gradient,bryson1961gradient}, via solving the adjoint system backward in time and updating the new control sequence, known as the adjoint method. As a second-order convergent algorithm, Differential Dynamic Programming (DDP) has been proposed to solve the value function and the control input via iteration of forward and backward path \cite{jacobson1970differential,liao1991convergence}. 
To mitigate the computational burden in DDP caused by a tensor calculus, an ``iterative LQR" (iLQR) method has been proposed in \cite{li2004iterative}, which successfully eliminates the heavy computation by invoking the linearization only at the backward path.

Among the many optimization methods, gradient descent is perhaps one of the most recognized approaches. The gradient-based methods for optimizing an instantaneous cost have been applied to planning tasks of the mobile robots, such as in \cite{morbidi2012active, wei2014optimized} for active target tracking, in \cite{cristofalo2019vision} for localizing 3-D features in an environment. For enabling the gradient calculus, a differentiable formulation of a field of view has been proposed in \cite{murali2019perception}. However, when applying the gradient descent to the optimal control with dynamical systems, the gradient with respect to multi-step control input sequences is needed, unlike optimizations of a static function or an instantaneous cost. The adjoint method has been widely utilized as an implicit gradient descent method for solving such an optimal control under multi-step control input sequences. However, the dynamical system in the adjoint method is supposed to belong to a vector field, which is not directly applicable to the $SE(3)$ pose kinematics we consider. In addition to it, an optimal control approach for an information-based active exploration and mapping remains an open problem. 

\textbf{Contributions:}
We develop \emph{iterative Covariance Regulation} (iCR), a new forward-backward gradient descent algorithm for finite-horizon optimal control of the covariance matrix of an active estimation process. To ensure that the Riccati equation governing the covariance evolution can be differentiated with respect to a control sequence, we introduce a new \emph{differentiable field of view} formulation of the sensing model.



\section{Problem Statement}

This section formalizes active exploration and mapping as an optimal control problem. 

\subsection{Motion and Sensor Models} 
Consider a robot with pose $T_k \in SE(3)$ at time $ t = t_k \in \R_+$, where $\{t_k\}_{k=0}^{K}$ for some $K \in {\mathbb N}$ is an increasing sequence. The definition of pose and its discrete-time kinematic model are:
  \begin{align} 
  T_k:= \left[ 
\begin{array}{cc} 
R_k & \bfp_k \\
{\mathbf 0}_{3 \times 1}^\top & 1 
\end{array} 
\right],  \quad \label{SE2dyn} 
 T_{k+1} =  T_{k} \exp \left(\tau \hat \bfu_k \right) , 
 \end{align}  
where $\bfp_k \in  \R^3$ is position, $R_k \in SO(3)$ is orientation, and $\bfu_k = [\bfv_k^\top, \bfomega_k^\top]^\top \in \R^6$ is a control input, consisting of linear velocity $\bfv_k \in \R^3$ and angular velocity $\bfomega_k \in \R^3$. The hat operator $ \hat{(\cdot)} : \R^6 \to se(3)$ maps vectors in $\R^6$ to the Lie algebra $se(3)$ associated with the $SE(3)$ Lie group \cite{barfoot2017state}. 
 
The robot evolves in a 3-D environment described by a set $\Omega \subset \R^3$. A map $\bfm \in \R^n$ of the environment is defined by discretizing $\Omega$ into $n \in {\mathbb N}$ tiles, such as voxels or octants \cite{hornung2013octomap}, and associating each tile $j \in {1,\ldots,n}$ with a position $\bfp^{(j)} \in \R^3$ and a mapped value $\bfm^{(j)} \in \R$, such as occupancy or signed distance \cite{oleynikova2017voxblox}.



Let $\bfz_k \in \R^n$ be a measurement of the environment obtained by the robot at time $t_k$, according to the following sensor model: 
\begin{align} \label{eq:sensor} 
    \bfz_k = \bfh(T_k, \bfm) + \bfeta_k, \qquad \bfeta_k \sim \N (0, V(T,\bfm)),
\end{align}                        
where $\bfh: SE(3) \times \R^n \to \R^{n}$ is the measurement function which depends on both the pose and map states, and $\bfeta_k$ is Gaussian sensing noise with zero mean and covariance matrix $V: SE(3) \times \R^n \to \R^{n \times n}$. We present a differentiable formulation of the noise covariance matrix, which allows capturing field of view constraints, in Sec. \ref{sec:planning}.

\subsection{Active Exploration and Mapping} 

Assuming a Gaussian prior on the map state $\bfm$, we construct its posterior conditioned on a sequence of measurements $\bfz_{1:k}$ by means of the Extended Kalman Filter (EKF):
\begin{align}
    \bfm | \bfz_{1:k} \sim \N\left(\bfmu_k, \Sigma_k \right) , 
\end{align}
where the mean $\bfmu_k \in \bbR^n$ and covariance $\Sigma_k \in \bbS_{\succ 0}^{n \times n}$ satisfy:
\begin{equation}\begin{aligned} 
 \bfmu_{k+1} &= \bfmu_k +  \Sigma_{k} H_{k+1}^\top R_{k+1}^{-1} (\bfz_{k+1} - h(T_{k+1}, \bfmu_{k})), \\
    \Sigma_{k+1} &= (\Sigma_k^{-1} + M_{k+1} )^{-1}, \\
    R_{k} &=  H_k \Sigma_{k-1} H_k^\top + V_k, \quad M_k := H_k^\top V_k^{-1} H_k, \\
    H_{k+1} &=  \fr{\pa h(T_{k+1},\bfm)}{\pa \bfm}\bigg |_{ \bfm = \bfmu_k}, V_{k+1} = V(T_{k+1}, \bfmu_k). \label{eq:Hk-def}
\end{aligned}\end{equation}
The Extended Information Filter (EIF) \cite{thrun2004simultaneous}, is an equivalent Bayesian filtering approach to the EKF, which parameterizes the Gaussian distributions in terms of an information matrix $Y_k \in \bbS_{\succ 0}^{n \times n}$ and information mean $\bfxi_k \in \R^n$ as follows:
\begin{align}
\bfm | \bfz_{1:k} \sim & \N\left(Y_k^{-1}\bfxi_k, Y_k^{-1} \right).
\end{align}
The update equations for $\bfxi_k$ and $Y_k$ are given by:
\begin{equation}\begin{aligned} 
\bfxi_{k+1} & = \bfxi_{k} + H_{k+1}^\top R_{k+1}^{-1} (\bfnu_{k+1} + H_{k+1}Y_k^{-1} \bfxi_k ), \\
\bfnu_{k+1} & = \bfz_{k+1} - h(T_{k+1},Y_k^{-1} \bfxi_{k}), \\
    Y_{k+1} & = Y_k + M_{k+1}. \label{eq:EIF-Skupdate} 
\end{aligned}\end{equation}


We formulate active exploration and mapping as a motion planning problem aiming to minimize the differential entropy\footnote{The differential entropy of a continuous random variable $Y$ with probability density function $p$ is defined as $\bbH(Y) := - \int p(y) \log p(y) dy$.} in the map state $\bfm$: 
\begin{align} \label{eq:Entropy} 
\min_{\bfu_0, \dots, \bfu_{K-1}} \bbH( \bfm | \bfz_{1:K}, T_{1:K}), 
\end{align} 
subject to the motion model in \eqref{SE2dyn} and the EIF update in \eqref{eq:EIF-Skupdate}. 

While, in general, the active mapping problem above is a stochastic optimal control problem, owing to the Gaussian distribution of $\bfm | \bfz_{1:k}$, it can be reduced to a deterministic optimal control problem, in which open-loop control policies are optimal \cite{Atanasov14ICRA}. In particular, the differential entropy in \eqref{eq:Entropy} is proportional to $\log \det (Y_K^{-1})$ for a Gaussian distribution, and thus the problem can be reformulated as maximizing the following terminal reward function:
\begin{align} \label{reward-def}
    r^{\bfU} =  \log \det \left(  Y_{K}^{\bfU} \right), 
\end{align}
where $\bfU = [\bfu_0^\top, \bfu_1^\top, \dots, \bfu_{K-1}^\top]^\top \in \R^{6 K}$ is a control sequence. With the help of the EIF, computing the information matrix in \eqref{eq:EIF-Skupdate} recursively, the terminal reward function can be expressed with respect to the trajectory of the pose state $T_{1:K}$. However, since the equations in \eqref{eq:Hk-def} are evaluated at each updated mean both in the EKF and the EIF, which depend on the stochastic measurements, offline computation of the motion planning is not possible. Instead, we evaluate them at an initial estimate of the map $\bfmu_0 \in \R^n$, by which $H$ and $V$ can be fixed offline. The problem is stated as follows.

\textbf{Problem} \label{prob:AM}
Given a prior Gaussian distribution over the map state $\bfm \sim \N(\bfmu_0, \Sigma_0)$ with a mean $\bfmu_0 \in \R^n$ and the covariance matrix $\Sigma_0 \in \bbS^{n\times n}_{\succ 0}$, obtain an open-loop control sequence $\bfU = [ \bfu_0^\top, \dots, \bfu_{K-1}^\top ]^\top \in \R^{6K}$ to solve the following deterministic optimal control problem: 
\begin{align} 
\max_{\bfU} \hspace{1mm} \log \det (Y_K^{\bfU}), \label{prob:reward}
\end{align} 
subject to 
\begin{align} \label{prob:model} 
T_{k+1} &= T_{k} \exp \left(\tau \hat \bfu_k \right) , \quad k = 0, \dots, K-1,
\\ Y_K &=  \Sigma_0^{-1} + \sum_{k=1}^{K} M(T_k), \label{prob:SK-sol} 
\\ M(T) &= H(T)^\top V(T, \bfmu_0)^{-1} H(T),  \label{prob:M-def}
\\ H(T) &=  \fr{\pa h(T, \bfm)}{\pa \bfm} \bigg|_{\bfm = \bfmu_0} . \label{prob:H-def} 
\end{align}

\section{Planning Method} \label{sec:planning} 

This section develops ``iterative Covariance Regulation (iCR)", a new $SE(3)$ trajectory optimization method to solve the active mapping optimal control problem. 



\subsection{Differentiable Field of View}

For a given robot pose state $T$ and the position $\bfp^{(j)}$ of the $j$-th map cell, we consider the robot body-frame coordinates of $\bfp^{(j)}$:
\begin{align}
    \bfq(T, \bfp^{(j)}) = Q T^{-1} \underline{\bfp}^{(j)} , \label{eq:qj-def}
\end{align}
where the projection matrix $Q$ and the homogeneous coordinates $\underline{\bfp}^{(j)}$ are defined as:
\begin{align} 
  Q = \left[ 
  \begin{array}{cc} 
  I_3 & {\mathbf 0}_{3 \times 1} 
  \end{array} 
  \right] \in \R^{3 \times 4}, \;\;
  \underline{\bfp}^{(j)}  = \left[ 
  \begin{array}{c} 
  \bfp^{(j)} \\ 1 
  \end{array} 
  \right]  \in \R^4.  
\end{align}
Let the field of view of the robot be described as a fixed region ${\mathcal F} \subset \R^3$ in the body frame. We formulate the sensor noise matrix $V$ so that the magnitude of the noise in the unobserved domain $\Omega / {\mathcal F}$ is approximately infinite. Hence, we formulate the measurement matrices as 
 \begin{align} 
     V(T, \bfm) &= \textrm{diag} (\{V_{jj}(T, \bfm)\}_{j=1}^{n}) \in \R^{n \times n},  \label{eq:VT-def}
 \end{align}
 where 
 an approximate expression of $V_{jj}(T, \bfm)$ for $j \in 1, \dots, n$ is 
 \begin{align}
     V_{jj}(T, \bfm) \approx \begin{cases} 
     \sigma^2 , \quad \textrm{if} \quad \bfq(T, \bfp^{(j)}) \in {\mathcal F}, \\
     \infty , \quad \textrm{if} \quad \bfq(T, \bfp^{(j)}) \notin {\mathcal F},  
     \end{cases} \label{eq:Viidef}
 \end{align}
 where $\sigma \in \R_+$ is a standard deviation of the sensor noise in the observed domain. For the sake of enabling the gradient descent in the next section, we need the approximate expression in \eqref{eq:Viidef} by a continuous and differential function with respect to $T$. For that reason, we introduce a Signed Distance Function (SDF) defined below for further analysis.

\begin{definition}
The \emph{signed distance function} $d : \R^3 \to \R $ associated with a set $\calF \subset \bbR^3$ is:
\begin{align}
     d(\bfq, {\mathcal F})  = 
     \begin{cases}
     - \min_{\bfq^* \in \pa {\mathcal F}}  || \bfq - \bfq^*||, \quad \textrm{if} \quad \bfq \in {\mathcal F}, \\
     \phantom{-} \min_{\bfq^* \in \pa {\mathcal F}} ||\bfq - \bfq^*||, \quad \textrm{if} \quad \bfq \notin {\mathcal F},  
     \end{cases}
\end{align}
where $\pa {\mathcal F}$ is the boundary of ${\mathcal F}$.
\end{definition}

The conditions $\bfq(T, \bfp^{(j)}) \in {\mathcal F}$ and $\bfq(T, \bfp^{(j)}) \notin {\mathcal F}$ in \eqref{eq:Viidef} are then replaced by inequality conditions in terms of the SDF of $\calF$. Incorporating this idea, the inverse of \eqref{eq:Viidef} is described as:
\begin{align}
     V_{jj}^{-1}(T, \bfm) \approx \fr{1}{\sigma^2} \begin{cases} 
      1, \quad \textrm{if} \quad d(\bfq(T, \bfp^{(j)}), {\mathcal F}) \leq 0, \\
     0 , \quad \textrm{if} \quad d(\bfq(T, \bfp^{(j)}), {\mathcal F})  > 0.  \label{V-inv-1st} 
     \end{cases} 
\end{align}
%
In order to approximate the right-hand side of \eqref{V-inv-1st} by a continuously differentiable function, we rely on a probit function \cite{bishop2006pattern}, defined by the Gaussian CDF $\Phi : \R \to [0, \;1]$:
\begin{align}
     \Phi(x) =& \frac{1}{2} \left[ 1 + \textrm{erf} \left(\frac{x}{\sqrt{2} \kappa} - 2  \right) \right] , \label{Phi-def} 
\end{align}
where $\textrm{erf}(y) := \fr{2}{\sqrt{\pi}} \int_0^{y} e^{-t^2} dt $. Then, \eqref{Phi-def} satisfies $\Phi(x) \approx 0 $ for all $x < 0$ (indeed, $\Phi(0) = 0.002...$ ), for all tuning parameter $\kappa >0$ which gives the smoothness of the function. Namely, we have $\lim_{\kappa \to +0} \Phi(x) = He(x)$, where $He(x)$ is the Heaviside function, also known as the unit step function. Using the function \eqref{Phi-def}, we formulate \eqref{V-inv-1st} as:
\begin{align}
     V_{jj}^{-1}(T, \bfm)  = \fr{1}{\sigma^2} \left(  1 - \Phi( d(\bfq(T, \bfp^{(j)}), {\mathcal F}))  \right)  ,  
          \label{V-inv-2} 
\end{align}
by which a differentiable field of view is obtained and we can apply the gradient descent method next. The visualization of the differentiable field of view is given in Fig.~\ref{fig:DFOV} which depicts 2-D plot of \eqref{V-inv-2} with respect to $\bfp^{(j)}$. Here, the field of view is set as a cone projected onto 2-D space, SDF of which is derived in Appendix.  
 
 \begin{figure}[t]
 \vspace{1.5mm}
 \centering
  \includegraphics[width=0.80\linewidth]{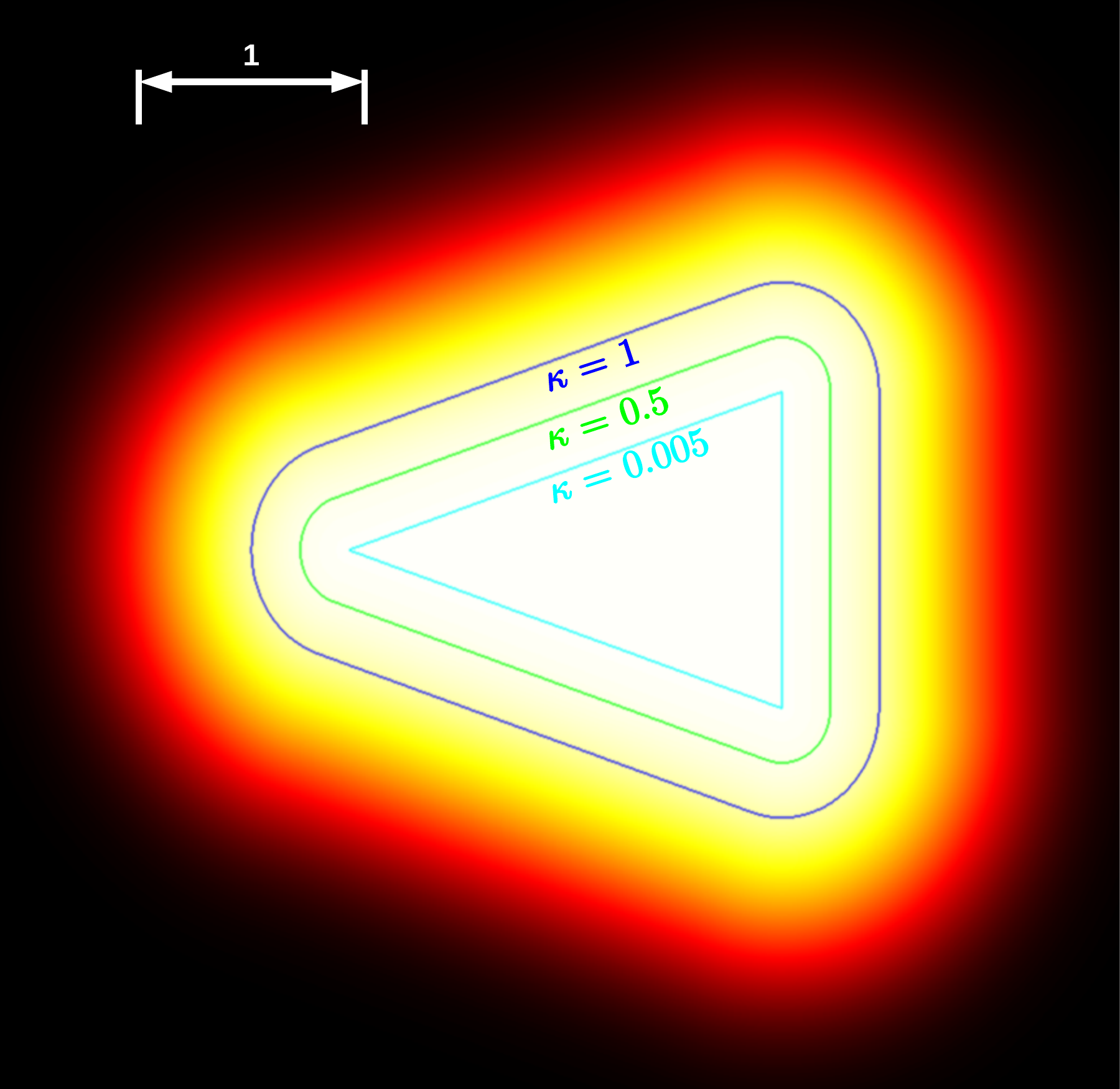}%
  \hfill%
  \caption{2-D plot of a differentiable field of view \eqref{V-inv-2} specified by the signed distance function of a 2-D cone with parameters $\sigma =1$ and $\kappa = 0.5$. The closed lines depict the level sets of $V^{-1}_{jj} = 0.98$ for different choices of $\kappa$. As $\kappa$ gets larger, the area of each level set becomes wider and $V^{-1}_{jj}$ varies smoothly.}
  \label{fig:DFOV}
\end{figure}

\subsection{Iterative  Covariance  Regulation}

\RestyleAlgo{boxruled}
\begin{algorithm}
	\caption{iterative Covariance Regulation (iCR)}
	\label{alg:icr}
	\SetKwBlock{Repeat}{repeat}{}
	\KwData{Initial robot pose $T_0 \in SE(3)$, initial map covariance $\Sigma_0 \in \bbS_{\succ 0}^{n \times n}$, map cell positions $\bfp^{(j)}$ for $j \in \{1,\ldots,n\}$, and initial control sequence $\bfu_{0:K-1}$} 
	Initialize $M \gets 0_{n \times n}$, $\fr{\pa M}{\pa u} \gets 0_{n \times n}$. \\
	\Repeat{
	Set $Y \gets \Sigma_0^{-1}$. \\
	\For{$k\gets 0$ \KwTo $K - 1$}{
    $T_{k+1} \gets T_k \exp \left( \tau  \hat \bfu_k \right)$. \\
    \For{$j\gets 1$ \KwTo $n$}{
    $M_{jj} \gets H(T)^\top V^{-1}(T) H(T)$ by RHS of \eqref{V-inv-2} using $[\bfp^{(j)}, T_{k+1}]$. 
    }
    Update $Y \gets Y + M$. 
    }
	\For{$k\gets K - 1$ \KwTo $0$}{
	Initialize $\fr{\pa r}{\pa u^{(i)}} \gets 0$ for all $i \in \{1, \dots, 6\}$. \\
	\For{$s\gets k + 1$ \KwTo $K$}{
	\For{$i\gets 1$ \KwTo $6$}{
		\uIf{$s = k+1$}{
          Set $\Lambda \gets T_k \fr{\pa \exp( \tau \hat \bfu_k)}{\pa u_k^{(i)}} $ by $[T_k, \bfu_k]$.
        }
        \Else{
            Update $ \Lambda  \gets  \Lambda \exp(\tau \hat \bfu_{s-1}) $. 
        }
	\For{$j\gets 1$ \KwTo $n$}{
	$[\bfq, \fr{\pa \bfq}{\pa u}] \gets $\eqref{eq:qj-def}, \eqref{eq:grad_q} by $[\bfp^{(j)}, T_k, \Lambda]$\!\!  \\
	$\fr{\pa M}{\pa u}_{jj} \gets $ RHS of \eqref{eq:r-def} by $[\bfq, \fr{\pa \bfq}{\pa u}]$. 
	}
	$\fr{\pa r}{\pa u^{(i)}} \gets \fr{\pa r}{\pa u^{(i)}} + \textrm{tr} \left(Y^{-1}  \fr{\pa M}{\pa u} \right) $. 
	}
    }
    $u_k^{(i)} \gets u_k^{(i)} + \alpha^{(i)} \fr{\pa r}{\pa u^{(i)}}$ for all $i \in \{1, \dots, 6\}$. 
    }
    }
\end{algorithm}

Let $u_k^{(i)} \in \R$ be the $i$-th element of the control vector $\bfu_k \in \R^6$ at time $t = t_k$. Our proposed method for active exploration and mapping is presented in Alg.~\ref{alg:icr}. Its derivation is based on gradient descent in the space of control sequences $\bfU$ with respect to the terminal reward $r^\bfU$.


\begin{proposition} \label{prop:gd}
The gradient-descent update for solving active exploration \eqref{prob:reward}--\eqref{prob:H-def} with differentiable field of view \eqref{eq:VT-def}, \eqref{V-inv-2} is given by
\begin{align}
\bfU \leftarrow \bfU + \Gamma \fr{\pa r^{\bfU}}{\pa \bfU}, 
\end{align}
where $\Gamma \in \R^{6K \times 6K }$ is a step size matrix given as $\Gamma := \diag(\{\{\gamma_k^{(i)}\}_{i=1}^{6}\}_{k=0}^{K-1})$ with a positive constant $\gamma_k^{(i)} \in \R_+$, and each element in the gradient $\fr{\pa r^{\bfU}}{\pa \bfU} := [\{\{\fr{\pa r^{\bfU}}{\pa u_{k}^{(i)}}\}_{i=1}^{6}\}_{k=0}^{K-1} ]$ is
\begin{align} 
\fr{\pa r^{\bfU}}{\pa u_{k}^{(i)}} 
 &= \sum_{s=k+1}^{K} \tr \left( (Y_K^{\bfU})^{-1}   \fr{\pa M(T_s)}{\pa u_{k}^{(i)}}\right) ,  \label{log-grad} \\
\fr{\pa M(T_s)}{\pa u_{k}^{(i)}} &=  \diag \left( \left\{w \left( \bfq(T_s, \bfp^{(j)}), \fr{\pa \bfq(T_s, \bfp^{(j)})}{\pa u_k^{(i)}} \right) \right\}_{j=1}^{n} \right), \notag
\end{align}
\begin{align}
     w \left (\bfq, \fr{\pa \bfq}{\pa u} \right ) &=  -  \fr{1}{\sigma^2} \Phi'(d(\bfq,{\mathcal F})) \fr{\pa d(\bfq, {\mathcal F}) }{\pa \bfq} \fr{\pa \bfq }{\pa u},  \label{eq:r-def} \\
      \fr{\pa \bfq(T_s, \bfp^{(j)})}{\pa u_k^{(i)}} 
      &= - Q T_s^{-1} \Lambda^{(i)}_s T_{s}^{-1} \underline{\bfp}^{(j)} \label{eq:grad_q},  
  \end{align}
  and $\Lambda^{(i)}_s := \fr{\pa T_s}{\pa u_k^{(i)}}$ is obtained for $s \in \{k+1, \dots, K\}$ via:
  \begin{equation} \label{eq:Lambda-ini} 
      \Lambda^{(i)}_{s} \!= \begin{cases}
        T_k \fr{\pa \exp \left( \tau \hat \bfu_k \right) }{\pa u_k^{(i)}}, & \text{if } s = k+1\\
        \Lambda^{(i)}_{s-1} \exp\left(\tau \hat \bfu_{s-1} \right), & \text{if } s \in \{k+2, \dots, K\}.
      \end{cases}
  \end{equation}
  %
\end{proposition}

\begin{proof}
Taking the gradient of the reward in \eqref{reward-def} with the help of \eqref{prob:SK-sol}, one can obtain
\begin{align} 
  \fr{\pa r^{\bfU}}{\pa u_{k}^{(i)}} 
  &=  \textrm{tr} \left( (Y_K^{\bfU})^{-1} \fr{\pa Y_K^{\bfU}}{\pa u_{k}^{(i)}} \right) \notag\\
  &= \textrm{tr} \left( (Y_K^{\bfU})^{-1} \left(  \sum_{s=1}^{K} \fr{\pa M(T_s)}{\pa u_{k}^{(i)}} \right) \right)
  \notag\\
 &= \sum_{s=k+1}^{K}\textrm{tr} \left( (Y_K^{\bfU})^{-1} \left(   \fr{\pa M(T_s)}{\pa u_{k}^{(i)}} \right) \right). 
\end{align}
Substituting \eqref{prob:M-def}, \eqref{eq:VT-def}, \eqref{V-inv-2} and using the chain rule, we get \eqref{log-grad}--\eqref{eq:r-def}. 

In \eqref{eq:r-def}, $\Phi'(d(\bfq,{\mathcal F}))$ and $ \fr{\pa d(\bfq, {\mathcal F}) }{\pa \bfq}$ can be obtained by the derivative of \eqref{Phi-def} and the gradient of the SDF for a given field of view, and thus it remains to compute the gradient $\fr{\pa \bfq }{\pa u_k^{(i)}}$. Taking the gradient of \eqref{eq:qj-def} with respect to $u_{k}^{(i)}$ yields \eqref{eq:grad_q}, by defining $\Lambda^{(i)}_s := \fr{\pa T_s}{\pa u_k^{(i)}}$. The initial condition and the update law for $\Lambda^{(i)}_s$ can be derived by taking the gradient of the pose kinematics in \eqref{SE2dyn}, which leads to \eqref{eq:Lambda-ini}, and hence we can deduce Proposition \ref{prop:gd}.
\end{proof} 

For the computation of $\fr{\pa \exp \left( \tau \hat \bfu_k \right) }{\pa u_k^{(i)}}$ in \eqref{eq:Lambda-ini}, we use the following lemma (see \cite{barfoot2017state} for a proof). 

\begin{lemma}
For $\bfu_k = [u_k^{(1)},\ldots,u_k^{(6)}]^\top \in \bbR^6$, it holds that
\begin{align}
 \frac{\pa \exp(\tau \hat{\bfu}_k)}{\pa  u_k^{(i)}}
 &= \tau ({\mathcal J}_L(\tau \bfu_k)\bfe_i)^{\wedge} \exp(\tau \hat{\bfu}_k), 
\end{align}
where $\bfe_i \in \R^6$ is the $i$-th unit vector, and ${\mathcal J}_L(\cdot)$ is the left Jacobian of $SE(3)$. 
\end{lemma}

\section{Evaluation}

We examine the efficacy of the proposed method in various simulated scenarios using 2-D occupancy maps of real-world environments. The robot moves according to the motion model in \eqref{SE2dyn} as $SE(2)$ dynamics, and its on-board sensor measures the occupancy of neighboring map cells inside its field of view. The field of view ${\mathcal F}$ is set as an isosceles triangle with height $3$ [m] the angle between the two legs equal to $60^{\circ}$ (see Fig.~\ref{fig:DFOV} for visualization and Appendix for derivation). The mapping process is performed as follows. First, we define the initial distribution of $\bfm$ as a multivariate Gaussian $\calN(\mu_0, \Sigma_0)$, where $\mu_0 = \boldsymbol{0} \in \mathbb{R}^{n}$ and $\Sigma_0 = \diag(100) \in \mathbb{R}^{n \times n}$. We encode measurements as $z^{(j)} = -1$ for free cells and $z^{(j)} = 1$ for occupied cells, both of which are for all $j$-th cells inside the field of view. Next, we construct a measurement vector $\bfz_k \in \R^n$, the $j$-th element of which has the measured occupancy value $z^{(j)}$ for all $j$-th cells inside the field of view. All other elements have a prior mean $\bfmu_{k-1}^{(j)}$ for all the $j$-th cells outside the field of view. The constructed measurement vector $\bfz_k$ is passed to EKF to get the updated mean $\bfmu_k$ and the information matrix $Y_k$. In order to visualize the occupancy map, we apply a threshold function $g: \R \to \{-1, 1\}$ which returns $g(x) = 1$ for all $x >0$ and $g(x) = -1$ for all $x <0$, to each element in the updated mean $\bfmu_k$. We set the smoothing factor $\kappa$ in \eqref{Phi-def} to $0.5$. One should note that a small $\kappa$ eliminates the impact of the space outside of the field of view on the policy optimization, while selecting a large $\kappa$ leads to neglecting the influence of the region inside the field of view. Both extreme cases result in vanishing gradients during the optimal control computation. Furthermore, since we seek for an open-loop policy, the planning horizon $K$ can be chosen arbitrarily large. However, a longer planning horizon tends to require more iterations for iCR in order to ensue an informative trajectory. Therefore, due to computational power constraints, we choose $K = 5$, which based on our experiments can reach to an informative trajectory in less than $10$ iterations. In each planning phase, we start by initializing a trajectory with constant linear velocity $v_k = 1.5$, and random angular velocity $w_k \sim U[-\pi/10, \pi/10]$ for $k = 0, \ldots, K - 1$ and use iCR at each planning in order to optimize the trajectory.

\begin{figure}[t]
\vspace{1.5mm}
  \centering
  \begin{subfigure}[t]{\linewidth}
  \centering
  \includegraphics[width=0.30\linewidth]{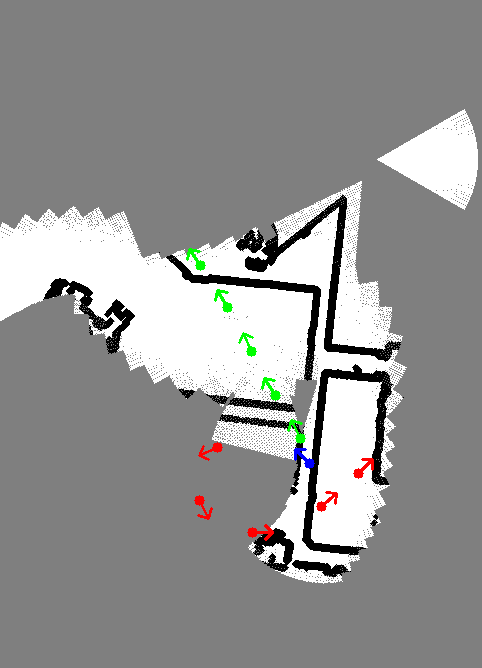}%
  \hfill%
  \includegraphics[width=0.30\linewidth]{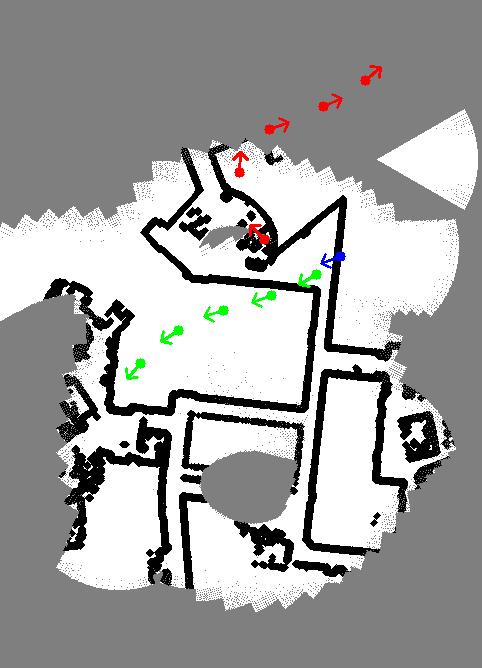}%
  \hfill%
  \includegraphics[width=0.30\linewidth]{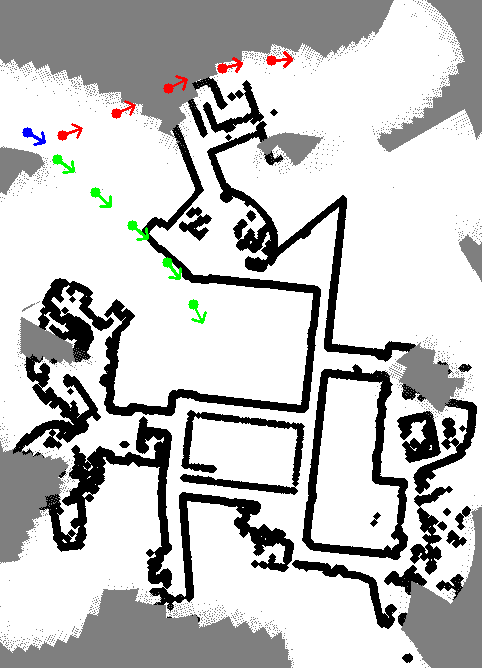}
  \end{subfigure}\\
  \begin{subfigure}[t]{\linewidth}
  \end{subfigure}\\
  \begin{subfigure}[t]{\linewidth}
  \centering
  \includegraphics[width=0.30\linewidth]{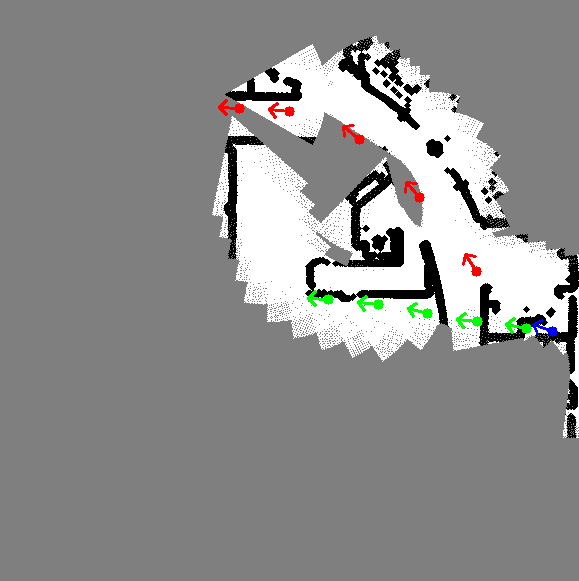}%
  \hfill%
  \includegraphics[width=0.30\linewidth]{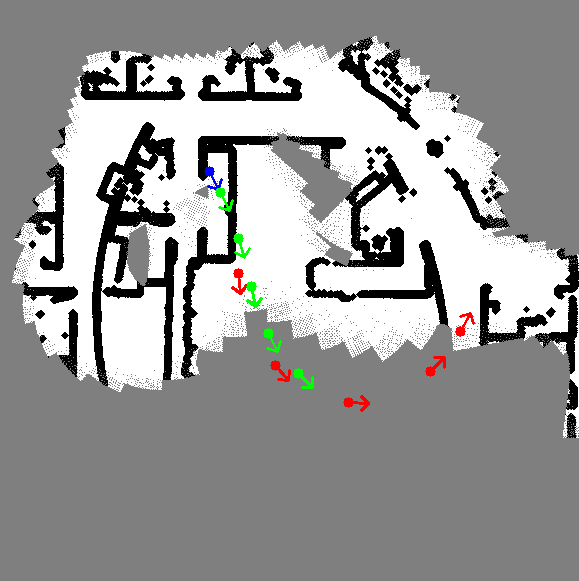}%
  \hfill%
  \includegraphics[width=0.30\linewidth]{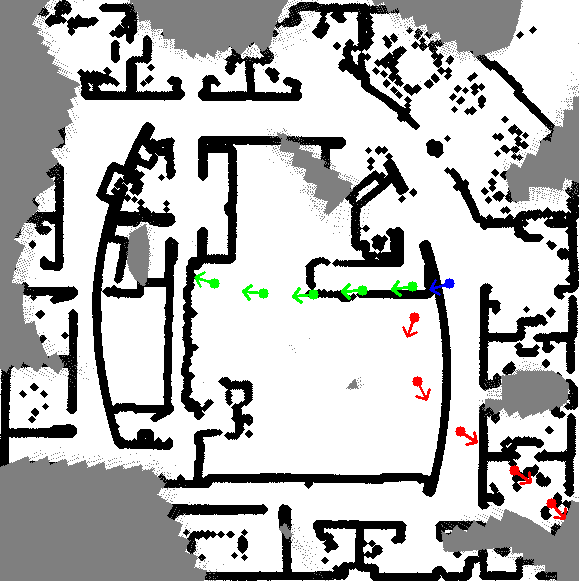}
  \end{subfigure}
  \caption{Snapshots of the robot's exploration and mapping using iCR in two environments, Environment A (top) and Environment B (bottom), at times $k = 51$ (left), $k = 126$ (center), and $k = 476$ (right), respectively. The occupancy maps are obtained by applying a binary threshold to the mean of the EKF. Green and red paths represent the initial trajectory and the optimal trajectory from iCR, respectively. The map data is available from C. Stachniss's lab \cite{freiburg}.}
\label{fig:sim_iCR}
\end{figure}

Fig.~\ref{fig:sim_iCR} shows snapshots of the robot's exploration and mapping via iCR for two environments. Environment A has dimensions $14.46 \text{[m]} \times 20.04 \text{[m]}$. Environment B has dimensions $17.37 \text{[m]} \times 17.43 \text{[m]}$. The robot tries to explore the center of both environments initially, where there is an abundance of unobserved map cells. Then, approximately after $k = 300$, the optimal trajectories move towards the edges of the environment, in which there remain some patches of unexplored regions. It is worth mentioning that iCR does not explicitly plan to visit unexplored regions; the observed behavior is only a result of minimizing the uncertainty, which leads to either refining the visited regions or discovering unexplored areas of the map. This is also evident by examining the information matrix of the EKF during the exploration episodes.

\begin{figure}[t]
\vspace{1.5mm}
  \centering
  \begin{subfigure}[t]{0.48\linewidth}
  \centering
  \includegraphics[width=0.48\linewidth]{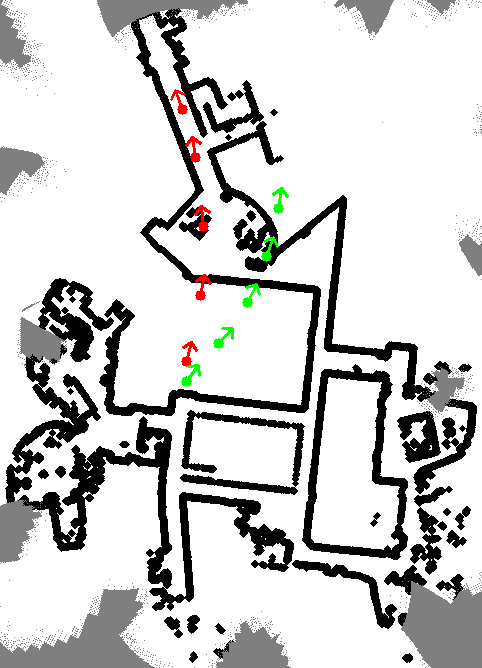}
  \hfill%
  \includegraphics[width=0.48\linewidth]{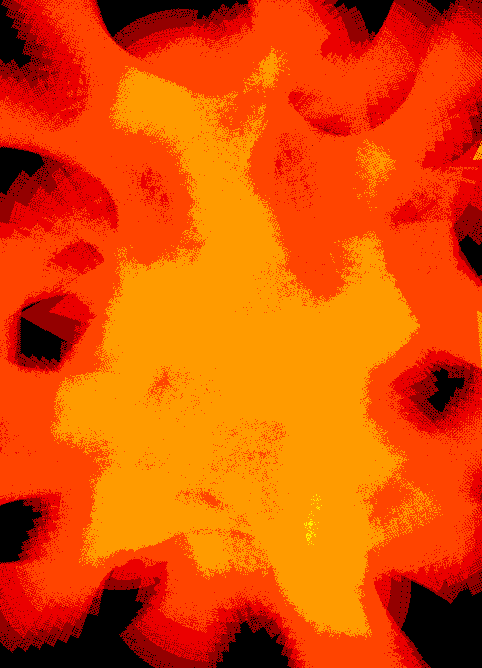}
  \caption{iCR}
  \end{subfigure} \hfill%
  \begin{subfigure}[t]{0.48\linewidth}
  \centering
  \includegraphics[width=0.48\linewidth]{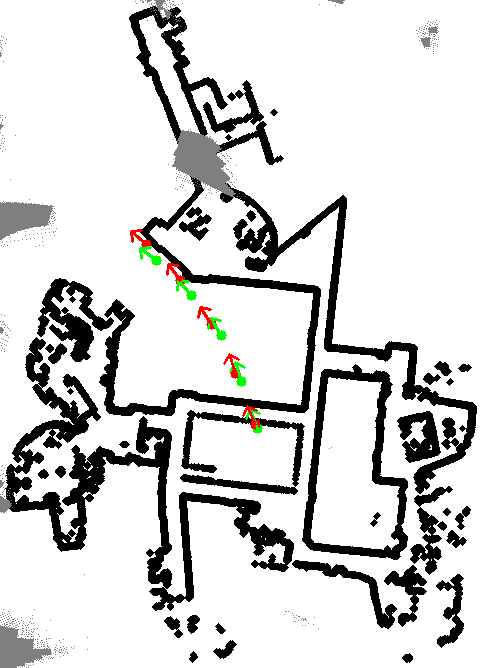}
  \hfill%
  \includegraphics[width=0.48\linewidth]{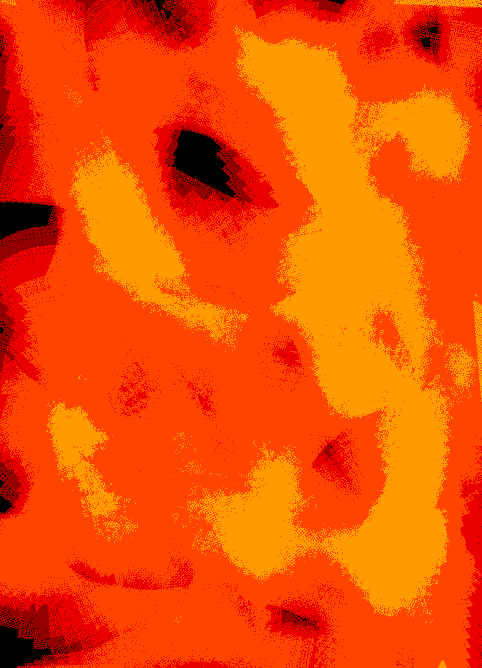}
  \caption{iCR + frontier}
  \end{subfigure}\\
  \begin{subfigure}[t]{0.48\linewidth}
  \centering
  \includegraphics[width=0.48\linewidth]{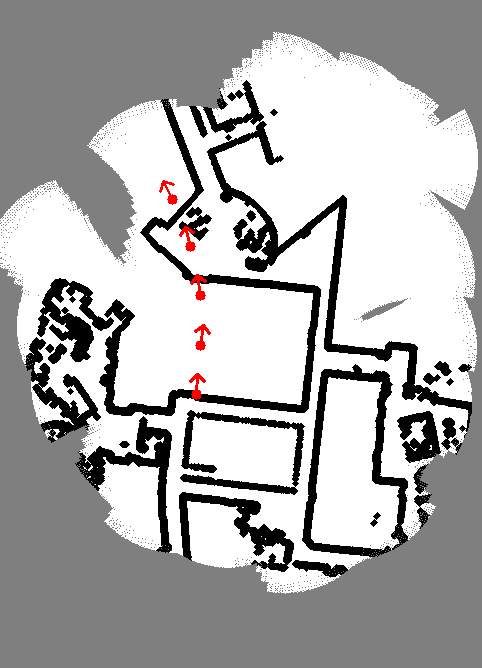}
  \hfill%
  \includegraphics[width=0.48\linewidth]{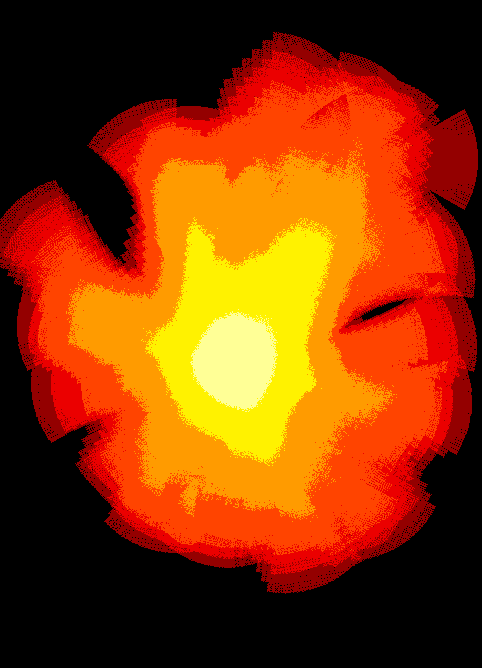}
  \caption{frontier}
  \end{subfigure} \hfill%
  \begin{subfigure}[t]{0.48\linewidth}
  \centering
  \includegraphics[width=0.48\linewidth]{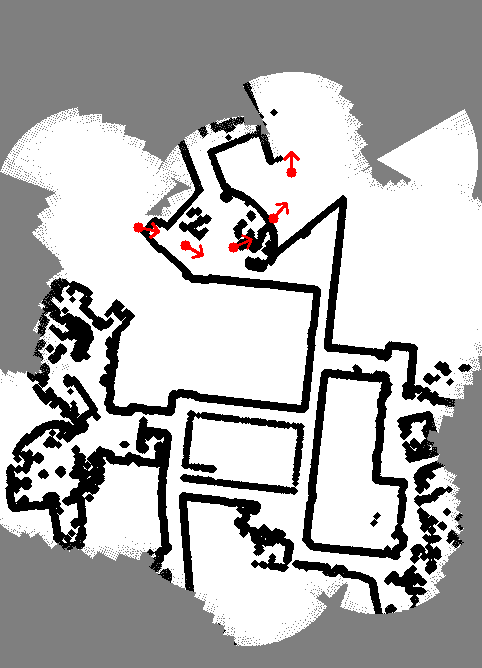}
  \hfill%
  \includegraphics[width=0.48\linewidth]{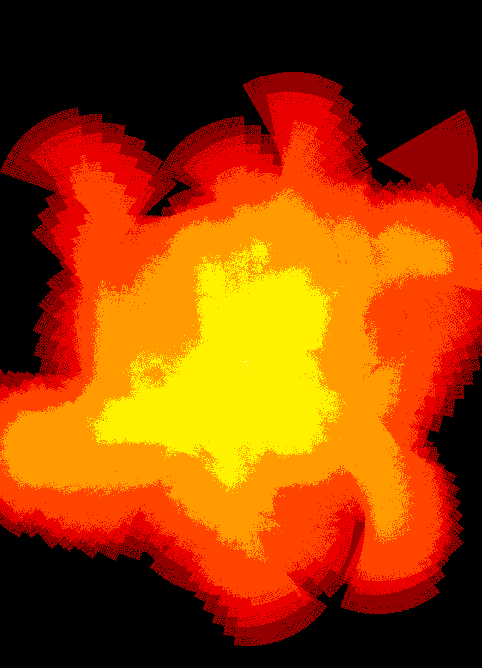}
  \caption{random}
  \end{subfigure}%
  \caption{Final map of Environment A at $k = 925$ for four different exploration strategies (a)--(d). For each pair in (a)--(d), the left image shows the occupancy map obtained by applying a threshold function to the mean of the EKF-based mapping, and the right image depicts the information heat map, obtained from the inverse covariance of the EKF.}
\label{fig:sim_map_6}
\end{figure}

\begin{figure}[t]
  \centering
  \begin{subfigure}[t]{0.48\linewidth}
  \centering
  \includegraphics[width=0.48\linewidth]{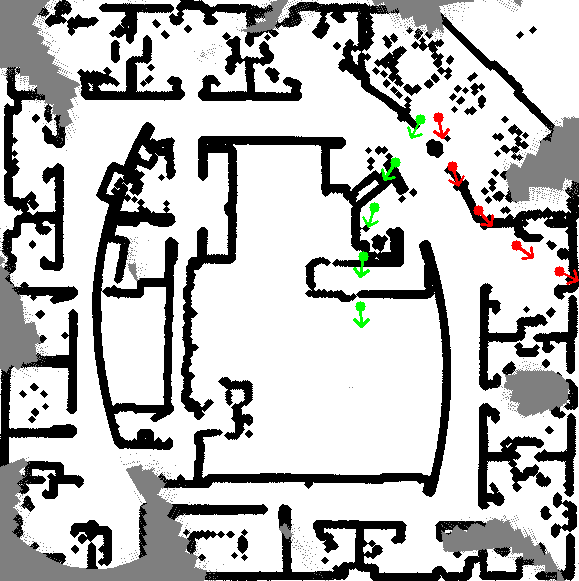}
  \hfill%
  \includegraphics[width=0.48\linewidth]{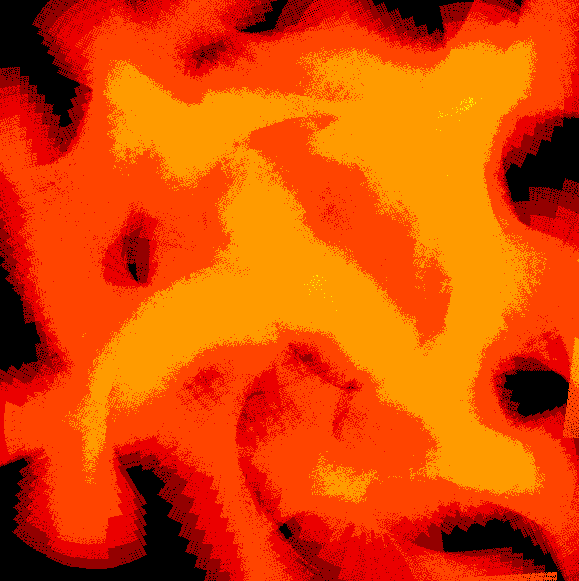}
  \caption{iCR}
  \end{subfigure} \hfill%
  \begin{subfigure}[t]{0.48\linewidth}
  \centering
  \includegraphics[width=0.48\linewidth]{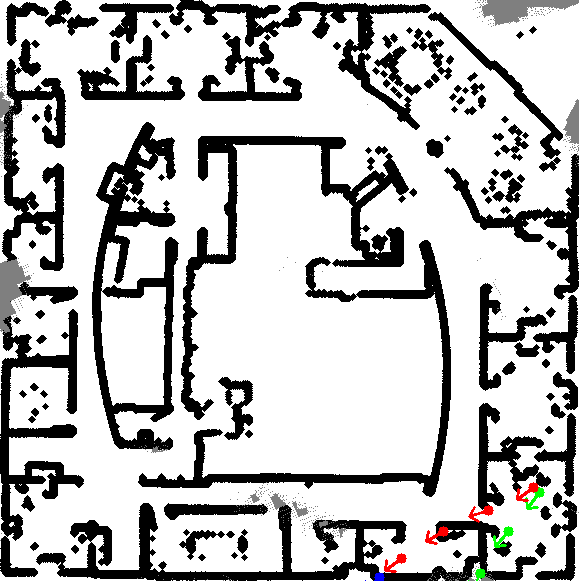}
  \hfill%
  \includegraphics[width=0.48\linewidth]{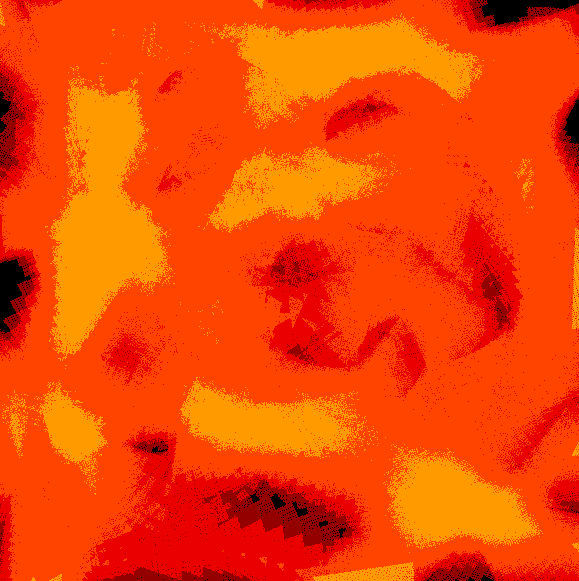}
  \caption{iCR + frontier}
  \end{subfigure}\\
  \begin{subfigure}[t]{0.48\linewidth}
  \centering
  \includegraphics[width=0.48\linewidth]{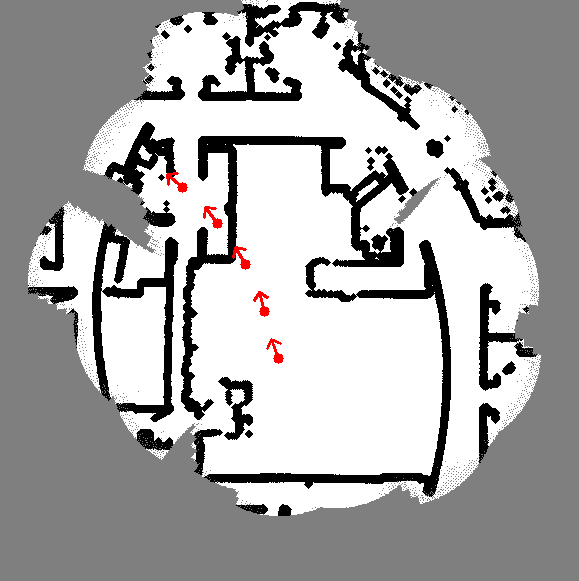}
  \hfill%
  \includegraphics[width=0.48\linewidth]{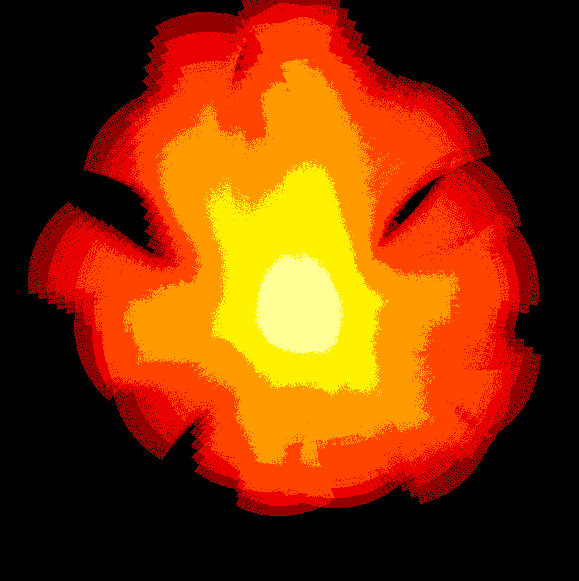}
  \caption{frontier}
  \end{subfigure} \hfill%
  \begin{subfigure}[t]{0.48\linewidth}
  \centering
  \includegraphics[width=0.48\linewidth]{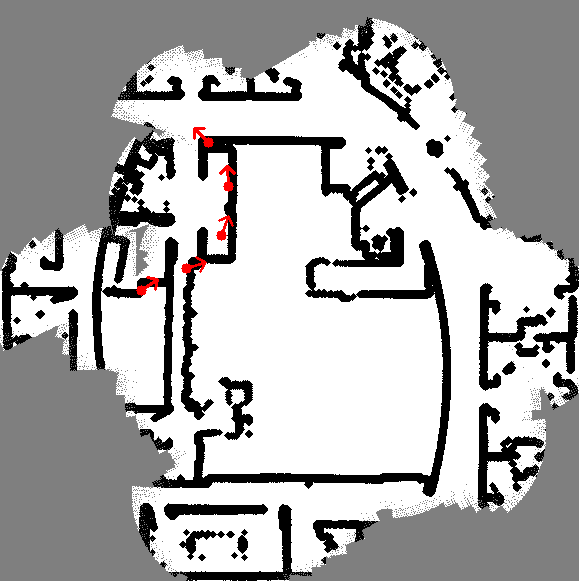}
  \hfill%
  \includegraphics[width=0.48\linewidth]{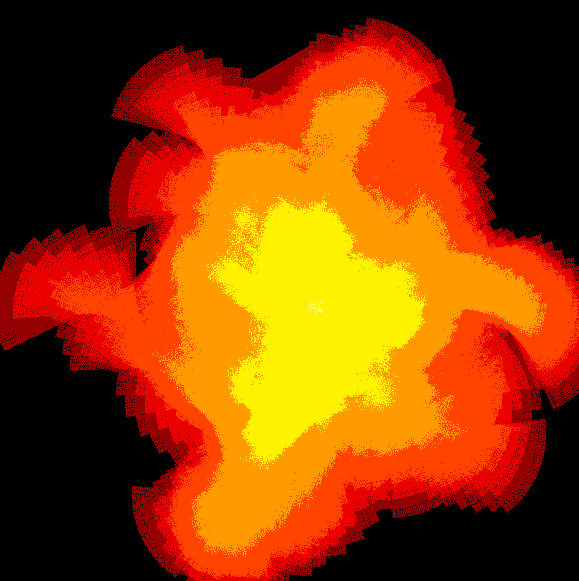}
  \caption{random}
  \end{subfigure}%
  \caption{Analogous demonstrations to Fig.~\ref{fig:sim_map_6} conducted for Environment B. }
\label{fig:sim_map_7}
\end{figure}

Fig.~\ref{fig:sim_map_6} illustrates the final occupancy map of Environment A at $k = 925$ and a heat map showing the diagonal elements in the information matrix obtained via four different exploration strategies: (a) iCR, (b) iCR + frontier, (c) frontier, and (d) random. Fig.~\ref{fig:sim_map_7} shows analogous plots performed in Environment B. As shown in Fig.~\ref{fig:sim_map_6} (a) and \ref{fig:sim_map_7} (a), via iCR, the center of the information map becomes brighter than the edges, meaning that the cells closer to the center receive more robot observations. Throughout multiple times of planning, we observe that the initial random path (green) brings insignificant uncertainty reduction since the trajectory overlaps the observed region highly. On the other hand, the optimal path from iCR (red) tends to visit unexplored areas, by which the information quantity from EKF becomes larger. This shows the advantage of iCR-based path planning which finds the informative trajectory over the continuous control space even by starting from a random control sequence without using any high-level heuristics, such as biasing the trajectory to visit frontiers as in \cite{yamauchi1998frontier, asgharivaskasi2021active, charrow2015information}. Furthermore, our simulations show that adding heuristics to the trajectory initialization in iCR can lead to more efficient exploration. For the demonstration, we set the initial trajectory during each planning step to face the boundary between observed map cells and the largest unexplored region. In this case, iCR can be considered as an augmentation over the path from frontier-based exploration \cite{yamauchi1998frontier}. As illustrated in both Fig.~\ref{fig:sim_map_6} (b) and Fig.~\ref{fig:sim_map_7} (b), the exploration via iCR with frontier receives relatively more observations in the edge map cells, compared to the exploration via iCR with random initialization does as shown in Fig.~\ref{fig:sim_map_6} (a) and Fig.~\ref{fig:sim_map_7} (a), by which we can deduce that providing a frontier location to the initial trajectory can enhance the map certainty by enabling more homogeneous exploration.

We test two baseline exploration strategies to compare with iCR. The first baseline is the same as frontier-based exploration \cite{yamauchi1998frontier}, where we plan a trajectory facing the frontier between the explored area and the largest unexplored region. As shown in Fig.~\ref{fig:sim_map_6} (c) and Fig.~\ref{fig:sim_map_7} (c), frontier-based exploration leaves most of edge pixels unexplored. By investigating through frontier-based exploration episodes, we observe that the robot shortly starts to oscillate between edges of the map since the largest unexplored patch perpetually changes among different corners of the simulation environment, and hence the robot fails to utilize time steps in order to efficiently explore the environment. The second baseline is a random policy under constant linear velocity $v_k = 1.5$ and random angular velocity $w_k \sim U[-\pi/3, \pi/3]$. Similar to the frontier-based policy, the resulting final map under the random policy leaves a large part of the environment unexplored, which is expected since the random policy does not aim to minimize the uncertainty of the map $\bfm$.

\begin{figure}[t]
  \centering
  \includegraphics[width=0.5\linewidth]{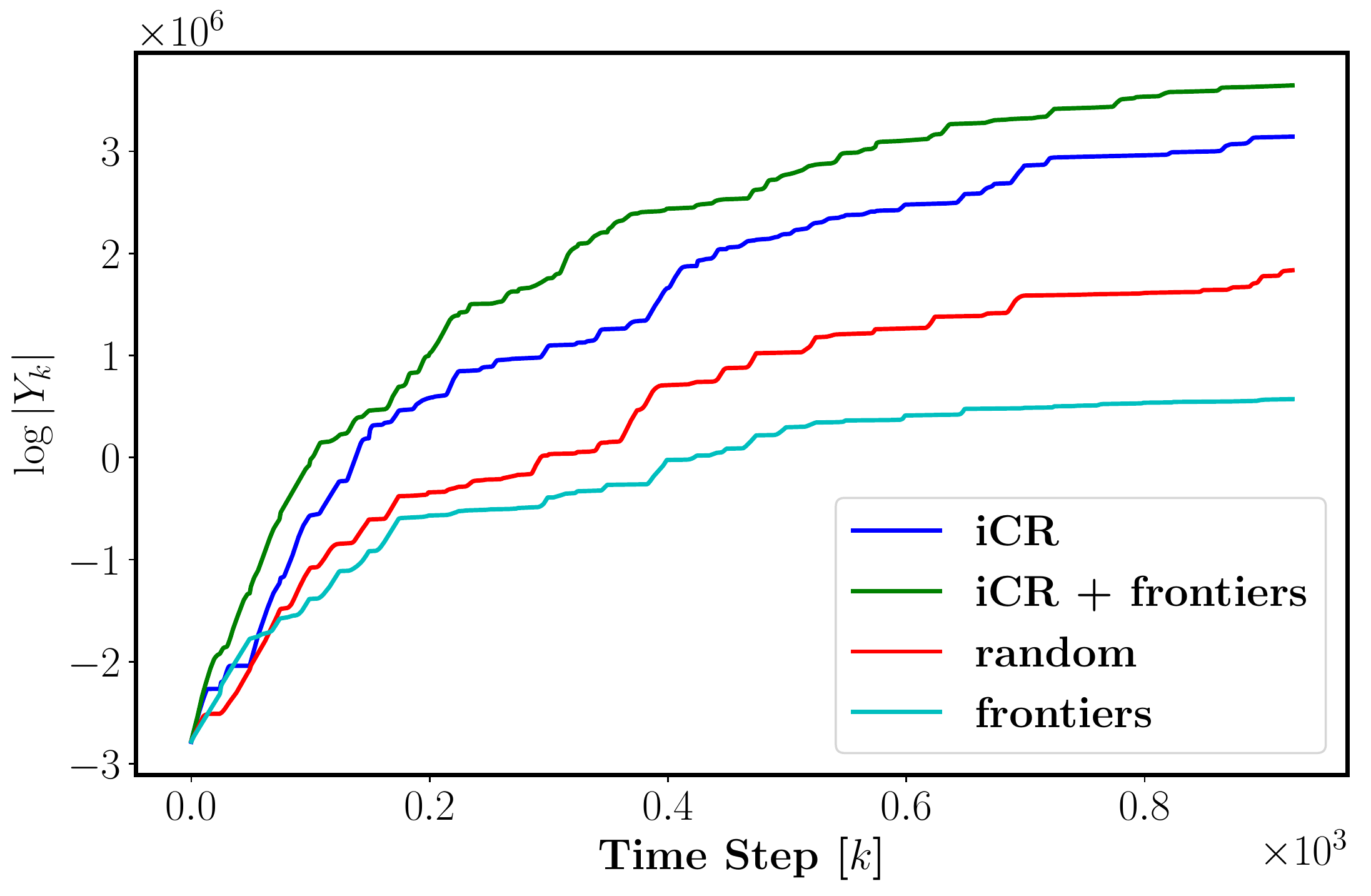}%
  \hfill%
  \includegraphics[width=0.5\linewidth]{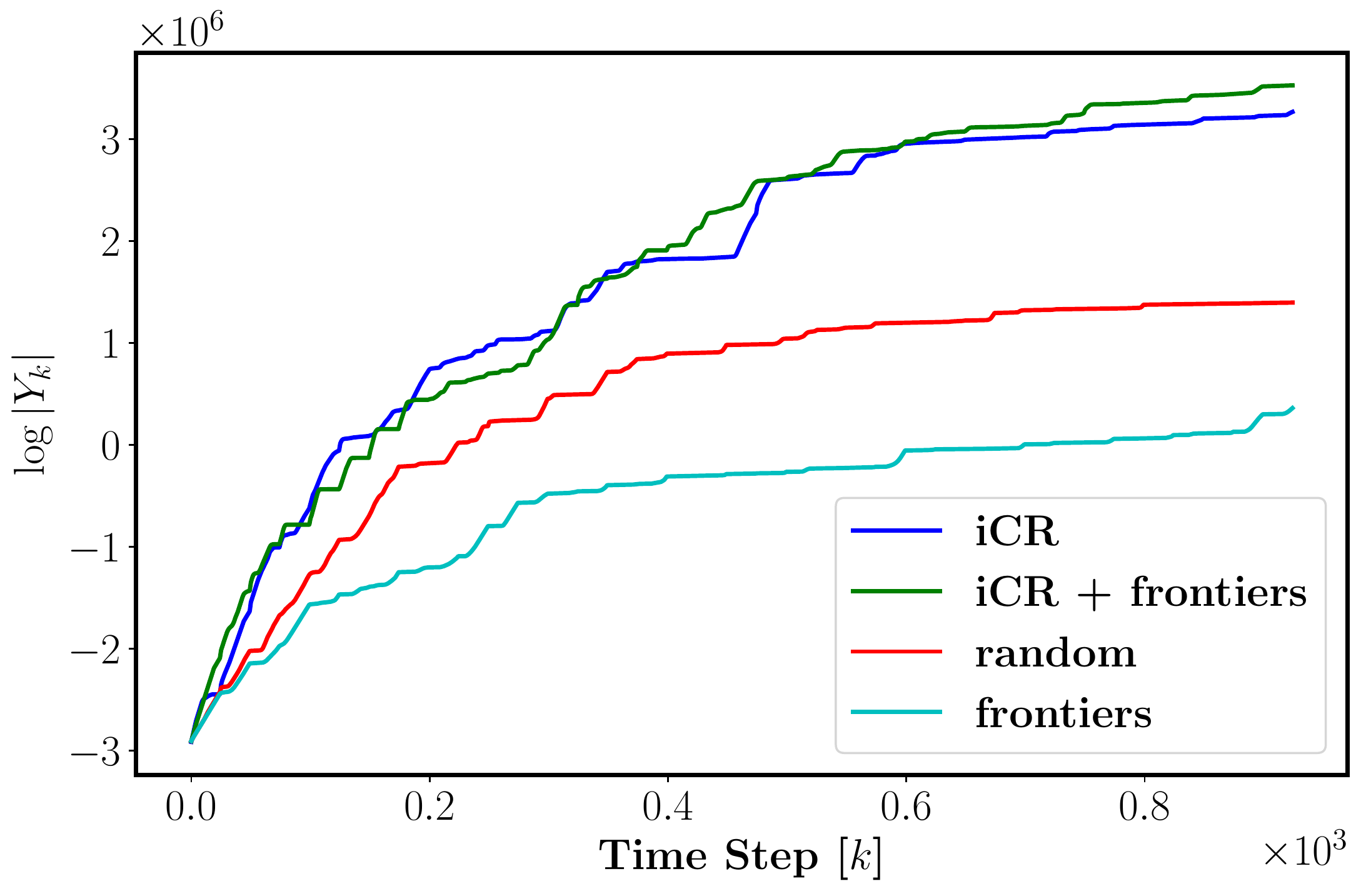}
  \caption{Evolution of the log-determinant of the information matrix over time steps. The left and right plots correspond to the exploration episodes of Fig.~\ref{fig:sim_map_6} and Fig.~\ref{fig:sim_map_7}, respectively.}
\label{fig:sim_perf}
\end{figure}

Fig.~\ref{fig:sim_perf} shows the reward function $\log |Y_{k}|$, the certainty in the map $\bfm$, over time for the experiments shown in Fig.~\ref{fig:sim_map_6} (left) and ~\ref{fig:sim_map_7} (right). In both plots in Fig.~\ref{fig:sim_perf}, we observe that both iCR (blue) and iCR + frontier (green) have a larger increase of the information than the information acquired by frontiers (cyan) and random policy (red) for all time. As mentioned earlier, the robot tends to start exploration by visiting information-rich regions, such as the center of the map. This results in a steep climb in the reward function during the initial phase of exploration ($k \leq 300$). As the robot makes more observations, the slope of the reward curve decreases, which corresponds to the refining phase where there are only a few unexplored map cells and the rest of the map has been discovered with high certainty. 

\section{Conclusion} 
This paper developed \emph{iterative Covariance Regulation} (iCR), a new forward-backward gradient descent algorithm for active exploration and mapping over continuous $SE(3)$ trajectories. Active mapping was posed as a finite-horizon deterministic optimal control problem, aiming to maximize the information matrix of the map at the terminal time conditioned on the potential measurement data from the on-board sensors. Our approach utilizes a differentiable field of view for the robot sensing model and developed a trajectory optimization approach by computing the gradient of the terminal reward function with respect to the multi-step control sequence over the $SE(3)$ manifold. The proposed method was demonstrated in the context of active occupancy grid mapping in numerical experiments. The paper considered occlusion-free and obstacle-free planning, which is clearly a limitation for real-world deployment. Future work will focus on accounting for occlusion in the differntiable field of view formulation and for obstacles in the iCR trajectory optimization to achieve safe active exploration and mapping.


\appendix[Signed Distance Function of a 2-D Cone]

Suppose that the field of view is a cone which has a vertex at the origin, the circular base at $x = h$, and the angles $[- \psi, \psi]$ from the vertex at $x-y$ plane. Then, the field ${\mathcal F}_{\rm c} \subset \R^3$ is given by 
 \begin{align*}
     {\mathcal F}_{\rm c} = \{(x,y,z) \in \Omega | y^2 + z^2 \leq (\tan(\psi) x)^2, \hspace{1mm} x \in [0, h] \}. 
 \end{align*}
Let ${\mathcal F}_{\rm c, z=\zeta}$ be the cone region projected onto 2-D space of $z = \zeta$ for a value $\zeta \in \R$. Then, the Signed Distance Function for ${\mathcal F}_{\rm c, z=0}$ is given in the following lemma. 
\begin{lemma}
 Signed Distance Function $d: \R^2 \times {\mathcal F}_{\rm c, z=0} \to \R$ of the cone projected onto 2-D space $z=0$ is given by 
 \begin{align}
     d(\bfq, {\mathcal F}_{\rm c, z=0}) = \begin{cases}
    \frac{ \bfa_i^\top \bfq + b_i}{|| \bfa_i||}, \quad \textrm{if} \quad  \bfq \in {\mathcal D}_i, \\
    || \bfq - \bfq_i ||, \quad \textrm{if}  \quad \bfq \in {\mathcal P}_i. 
     \end{cases} \label{app:SDF-cone} 
 \end{align}
 where 
 \begin{align*}
    p^*_{x} =& \fr{h}{  1 + \sin(\psi)}, \\
     {\mathcal D}_1 = & \{(x,y) \in \Omega | y \in (\underline{l}_1(x), \bar{l}_1(x)), \forall x \in (-\infty, h] \} , \\
     {\mathcal D}_2 = & \{(x,y) \in \Omega | y \in (-\bar{l}_1(x), - \underline{l}_1(x)), \forall x \in (-\infty, h] \} , \\
     {\mathcal D}_3 = & \{(x,y) \in \Omega | y \in (-\underline{l}_1(x), \underline{l}_1(x)), \forall x \in [p^*_x,  \infty ) \} 
 \end{align*}
 \begin{align*}
     {\mathcal P}_1 = & \{(x,y) \in \Omega |  y \in ( \bar{l}_1(x), +\infty), \forall x \in \R \} , \\
     {\mathcal P}_2 = & \{(x,y) \in \Omega | y \in (- \infty,  -\bar{l}_1(x)), \forall x \in \R \}, \\
     {\mathcal P}_3 = & \{(x,y) \in \Omega | y \in (-\underline{l}_1(x), \underline{l}_1(x)), \forall x \in \R_{\leq 0} \},
 \end{align*}
 \begin{align*}
     \underline{l}_1(x) =& \begin{cases}
     - \frac{1}{\tan(\psi)}x, \quad \forall x \in (-\infty, 0], \\
     0, \quad \forall x \in [0, p^*_x],\\
     \tan(\pi / 4 + \psi / 2) x - \fr{ h}{\cos (\psi)}, \hspace{1mm} \forall x \in [p^*_x, h],\\
     h \tan(\psi), \quad \forall x \in [h, \infty),
     \end{cases} \\
     \bar{l}_1(x) =& \begin{cases}
     - \frac{x - h}{\tan(\psi)} + h \tan(\psi), \quad \forall x \in (-\infty, h],\\
     h \tan(\psi), \quad \forall x \in [h, \infty),
     \end{cases}
 \end{align*}
 \begin{align}
    a_1 =& \left[
    \begin{array}{c}
    -1 \\
    \frac{1}{\tan(\psi)} 
    \end{array}
    \right], 
    a_2 = \left[
    \begin{array}{c}
    -1  \\
    -\frac{1}{\tan(\psi)} 
    \end{array}
    \right], 
    a_3 = \left[
    \begin{array}{c}
    1  \\
    0 
    \end{array}
    \right], \label{app:a-def} \\
    b_1 =& 0, \quad b_2 = 0, \quad  b_3 = -h, \label{app:b-def}\\
    q_1 =& \left[
    \begin{array}{c}
    h \\
    h \tan(\psi) 
    \end{array}
    \right],  
    q_2 = \left[
    \begin{array}{c}
    h  \\
    - h \tan(\psi) 
    \end{array}
    \right], 
    q_3 = \left[
    \begin{array}{c}
    0  \\
    0 
    \end{array}
    \right].\notag
\end{align}
 \end{lemma}
 
 \begin{proof}
   The separated domains ${\mathcal D}_i$ and ${\mathcal P}_{i}$ for $i \in \{1, 2, 3\}$ are defined so that if $\bfq \in {\mathcal D}_i$ then the closest point in $\pa {\mathcal F}$ lies in a line $\{\bfp \in \R^2 | \bfa_{i}^\top \bfp + b_i = 0\}$, and if $\bfq \in {\mathcal P}_i$ then the closest point in $\pa {\mathcal F}$ is a point $\bfq_i$ which is an edge of the projected cone. The normed distance $d^+: \R^2 \to \R_+$ between a point $\bfq$ and a line ${\mathcal L} := \{\bfp \in \R^2 | \bfa^\top \bfp + b = 0\}$ is known as 
   \begin{align*}
       d^+(\bfq, {\mathcal L}) = \fr{ | \bfa^\top \bfq + b |}{|| \bfa ||}. 
   \end{align*}
   There are two choices for a pair of $(\bfa, b)$ for any given line, in which the other pair has opposite sign to one. By setting the pairs $(\bfa_i, b_i)$ for $i \in \{1,2,3\}$ as \eqref{app:a-def}--\eqref{app:b-def}, one can deduce that if $\bfq \in {\mathcal F}_{\rm c}$ then $\bfa_i^\top \bfq + b_i \leq 0$ and if $\bfq \notin {\mathcal F}_{\rm c}$ then $\bfa_i^\top \bfq + b_i > 0$, which leads to the first line in \eqref{app:SDF-cone}. Indeed, all the other domains ${\mathcal P}_i$ is outside the cone ${\mathcal F}_{\rm c}$, thus if $\bfq \in {\mathcal P}_{i}$ then the SDF has a positive distance between the two points $\bfq$ and $\bfq_i$, which yields the second line of \eqref{app:SDF-cone}.  
 \end{proof}



\bibliographystyle{IEEEtran}
\bibliography{BIB_IROS21.bib}

\end{document}